\documentclass[11pt]{article}
\pdfoutput=1
\usepackage[round]{natbib}
\usepackage{hyperref}
\hypersetup{
    colorlinks=true,
    citecolor=bluegray,
    linkcolor=darkbrown,      
    urlcolor=blue,
}

\usepackage{times}
\usepackage{mdframed}
\usepackage{color}
\usepackage{graphicx}
\usepackage{subfigure}
\usepackage{enumitem}
\usepackage{amsmath}
\usepackage{amsthm}
\usepackage{amssymb}
\usepackage{array}
\usepackage{multirow}
\usepackage{caption}
\usepackage{bm}
\usepackage{bbm}
\usepackage{tabularx, ragged2e, booktabs}

\usepackage{fullpage}

\newcolumntype{C}{>{\Centering\arraybackslash}X}

\title{A Statistical Investigation of Long Memory in Language and Music}
\author{Alexander Greaves-Tunnell and Zaid Harchaoui \vspace{0.2cm}\\ University of Washington \\ \{alecgt, zaid\}@uw.edu}
\newtheorem{theorem}{Theorem}[section]
\newtheorem{proposition}[theorem]{Proposition}
\newtheorem{definition}[theorem]{Definition}

\usepackage{color}
\definecolor{puorange}{rgb}{0.80,0.20,0}
\definecolor{bluegray}{rgb}{0.04,0,0.7}
\definecolor{greengray}{rgb}{0.05,0.50,0.15}
\definecolor{darkbrown}{rgb}{0.40,0.2,0.05}
\definecolor{darkcyan}{rgb}{0,0.4,1}
\definecolor{black}{rgb}{0,0,0}
\definecolor{grey}{rgb}{0.93,0.93,0.93}

\newcommand{\norm}[1] {\left \| #1 \right \|}

\newcommand{\Tr}{\operatorname{\bf Tr}}

\DeclareMathOperator{\Cov}{Cov}
\DeclareMathOperator{\Var}{Var}

\DeclareMathOperator*{\argmin}{arg\,min}

\begin{document}
\maketitle

\begin{abstract}
Representation and learning of long-range dependencies is a central challenge confronted in modern applications of machine learning to sequence data. Yet despite the prominence of this issue, the basic problem of measuring long-range dependence, either in a given data source or as represented in a trained deep model, remains largely limited to heuristic tools. We contribute a statistical framework for investigating long-range dependence in current applications of deep sequence modeling, drawing on the well-developed theory of long memory stochastic processes. This framework yields testable implications concerning the relationship between long memory in real-world data and its learned representation in a deep learning architecture, which are explored through a semiparametric framework adapted to the high-dimensional setting.

\end{abstract}

\section{Introduction}
	Advances in the design and optimization of deep recurrent neural networks (RNNs) have lead to significant breakthroughs in the modeling of complex sequence data, including natural language and music. An omnipresent challenge in these sequence modeling tasks is to capture long-range dependencies between observations, and a great variety of model architectures have been developed with this objective explicitly in mind. However, it can be difficult to assess whether and to what extent a given RNN has learned to represent such dependencies, that is, whether it has \emph{long memory}. 

Currently, if a model's capacity to represent long-range dependence is measured at all, it is typically evaluated heuristically against some task or tasks in which success is taken an indicator of ``memory" in a colloquial sense. Though undoubtedly helpful, such heuristics are rarely defined with respect to an underlying mathematical or statistical property of interest, nor do they necessarily have any correspondence to the data on which the models are subsequently trained. In this paper, we pursue a complementary approach in which long-range dependence is assessed as a quantitative and statistically accessible feature of a given data source. Consequently, the problem of evaluating long memory in RNNs can be re-framed as a comparison between a learned representation and an estimated property of the data.

The main contribution is the development and illustration of a methodology for the estimation, visualization, and hypothesis testing of long memory in RNNs, based on an approach that mathematically defines and directly estimates long-range dependence as a property of a multivariate time series. We thus contextualize a core objective of sequence modeling with deep networks against a well-developed but as-yet-unexploited literature on long memory processes. 

We offer extensive validation of the proposed approach and explore strategies to overcome problems with hypothesis testing for long memory in the high-dimensional regime. We report experimental results obtained on a wide-ranging collection of real-world music and language data, confirming the (often strong) long-range dependencies that are observed by practitioners. However, we find that this property is not adequately captured by a variety of RNNs trained to benchmark performance on a language dataset..\footnote{Code corresponding to these experiments, including an illustrative Jupyter notebook, is available for download at \url{https://github.com/alecgt/RNN_long_memory}.} 


\paragraph{Related work.}

Though a formal connection to long memory processes has been lacking thus far, machine learning applications to sequence modeling have long been concerned with the capture of long-range dependencies. The development of RNN models has been strongly influenced by the identification of the ``vanishing gradient problem" in \cite{bengio}. More complex recurrent architectures, such as long short-term memory \citep{hochreiter}, gated recurrent units \citep{cho}, and structurally constrained recurrent networks \citep{mikolov} were designed specifically to alleviate this problem. Alternative approaches have pursued a more formal understanding of RNN computation, for example through kernel methods \citep{lei}, by means of ablative strategies clarifying the computation of the RNN hidden state \citep{levy}, or through a dynamical systems approach \citep{miller}. A modern statistical perspective on nonlinear time series analysis is provided in \citep{douc}.

Long-range dependence is most commonly evaluated in RNN models by test performance on a synthetic classification task. For example, the target may be the parity of a binary sequence (so-called ``parity" problems), or it may be the class of a sequence whose most recent terms are replaced with white noise (``2-sequence" or ``latch" problems) \citep{bengio,bengio2,lin}. A simple demonstration relatively early in RNN history by \citet{hochreiter2} showed that such tasks can often be solved quickly by random parameter search, casting doubt on their informativeness. Whereas the authors proposed a different heuristic, we seek to re-frame the problem of long memory evaluation so that it is amenable to statistical analysis.

Classical constructions of long memory processes \citep{mandelbrot2,granger,hosking} laid the foundation for statistical methods to estimate long memory from time series data. See also \citep{moulines, reisen} for recent works in this area. The multivariate estimator of \citet{shimotsu} is the foundation of the methodology we develop here. It is by now well understood that failure to properly account for long memory can severely diminish performance in even basic estimation or prediction tasks. For example, the sample variance is both biased and inefficient as an estimator of the variance of a stationary long memory process \citep{percivalgut}. Similarly, failure to model long memory has been shown to significantly harm the predictive performance of time series models, particularly in the case of multi-step forecasting \citep{brodsky}.

	\section{Background}
\paragraph{Long memory in stochastic processes.}
Long memory has a simple and intuitive definition in terms of the autocovariance sequence of a real, stationary stochastic process $X_t \in \mathbb{R}, t\in \mathbb{Z}$. The process $X_t$ is said to have long memory if the autocovariance
$$ \gamma(k) = \Cov(X_t, X_{t+k}), \ \ k \in \mathbb{Z}$$
satisfies
\begin{align} \gamma_X(k) \sim L_\gamma(k) \lvert k \rvert^{-(1-2d)} \ \ \text{as} \ \ k\to\infty, \label{eq:lm_acv} \end{align}
for some $d \in (0,1/2)$, where $a(k) \sim b(k)$ indicates that $a(k)/b(k) \to 1 \ \ \text{as} \ \ k\to\infty$ and $L_\gamma(k)$ is a slowly varying function at infinity (refer to Appendix A for an introduction to slowly varying functions). The term ``long memory" is justified by the slow (hyperbolic) decay of the autocovariance sequence, which enables meaningful information to be preserved between distant observations in the series. As a consequence of this slow decay, the partial sums of the absolute autocovariance sequence diverge. This can be directly contrasted with the ``short memory" case, in which the autocovariance sequence is absolutely summable. Moreover, we note that the parameter $d$ allows one to quantify the memory by controlling the strength of long-range dependencies.

In the time series literature, a spectral definition of ``memory" is preferred, as it unifies the long and short memory cases. A second-order stationary time series can be represented in the frequency domain by its spectral density function
$$ f_X(\lambda) = \sum_{k=-\infty}^\infty \gamma(k) e^{-ik\lambda}.$$
If $X_t$ has a spectral density function that satisfies
\begin{align} f_X(\lambda) &= L_f(\lambda) \lvert \lambda\rvert^{-2d} \label{eq:lm_sdf} \end{align}
where $L_f(\lambda)$ is slowly varying at zero, then $X_t$ has long memory if $d \in (0,1/2)$, short memory for $d = 0$, and ``intermediate memory" or ``antipersistence" if $d \in (-1/2,0)$. The two definitions of long memory are equivalent when $L_f(\lambda)$ is quasimonotone \citep{beran}. 

We summarize the complementary time and frequency domain views of long memory with a simple illustration in Figure \ref{fig:longmem_ex}, which contrasts a short memory autoregressive (AR) process of order 1 with its long memory counterpart, the fractionally integrated AR process. The autocovariance series is seen to converge rapidly for the AR process, whereas it diverges for the fractionally integrated AR process. Meanwhile, Eq. \eqref{eq:lm_sdf} implies that the long memory parameter $d$ has a geometric interpretation in the frequency domain as the slope of $\log f_X(\lambda)$ versus $-2\log(\lambda)$ as $\lambda \to 0$ (i.e. $-2\log(\lambda) \to \infty$).

\begin{figure*}
  \includegraphics[width=\textwidth,height=5cm]{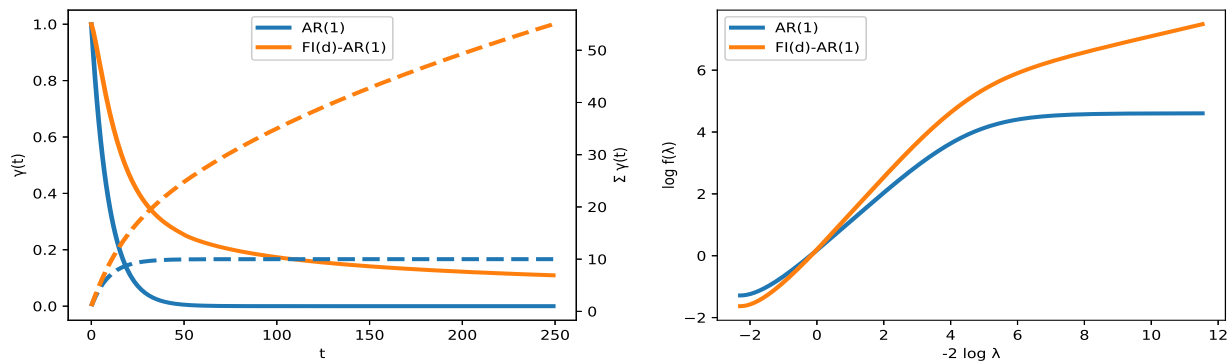}
  \caption{Time and frequency domain views of an AR(1) process $(\text{blue}, d = 0)$ and its long memory counterpart obtained by fractional differencing $(\text{orange}, d = 0.25)$. \emph{Left:} Autocorrelation sequences (solid lines) of the two processes, along with their partial sums (dotted lines). \emph{Right:} Plot of log spectral density versus $-2$ times log frequency.}
  \label{fig:longmem_ex}
\end{figure*}

Finally, the relevance of this topic in machine learning can be motivated through its consequences for prediction. Again, the contrast between the AR and fractionally integrated AR processes provides a simple and concrete illustration. Writing each process in terms of its Wold decomposition \citep{davis}
$$X_t = \sum_{j=0}^\infty a_j \varepsilon_{t-j},$$
the optimal predictor $\hat{X}_{t+h}$ of $X_{t+h}$ for the prediction horizon $h \geq 1$ is defined as the minimizer of the mean squared error $\text{MSE}(h) = \mathbb{E}[(\hat{X} - X_{t+h})^2]$
$$\hat{X}_{t+h} = \argmin_{\hat{X}} \mathbb{E}[(\hat{X} - X_{t+h})^2] = \sum_{j=0}^\infty a_{j+h} \varepsilon_{t-j}$$
and analyzed in terms the proportion of variance explained $R^2(h) = 1 - \Var(X_{t+h})^{-1} \text{MSE}(h)$. In the short memory AR case, the terms $\vert a_j\vert$, and therefore $R^2(h)$, decay exponentially (see Appendix B); by contrast, when $X_t$ is fractionally integrated $R^2(h)$ decays only hyperbolically \citep{beran}. Therefore, in the long memory regime, past observations can retain significant explanatory power with respect to future prediction targets, and informative forecasts are available over horizons extending well beyond that of an analogous short memory process.

\paragraph{Many common models do not have long memory.}

Despite the appeal and practicality of long memory for modeling complex time series, we emphasize that it is absent from nearly all common statistical models for sequence data. We offer a short list of examples, with proofs deferred to Appendix B of the Supplement.
\begin{itemize}

\item \emph{Markov models}. If $X_t$ is a Markov process on a finite state space $\mathcal{X}$, and $Y_t = g(X_t)$ for any function $g: \mathcal{X} \to \mathbb{R}$, then $Y_t$ has short memory. We show that this property holds even in a complex model with Markov structure, the Markov transition distribution model for high-order Markov chains \citep{raftery}. In light of this, long memory processes are sometimes called ``non-Markovian".

\item \emph{Autoregressive moving average (ARMA) models}. ARMA models, a ubiquitous tool in time series modeling, likewise have exponentially decaying autocovariances and thus short memory \citep{davis}. This may be somewhat surprising, as causal ARMA models with nontrivial moving average components are equivalent to linear autoregressive models of infinite order.

\item \emph{Nonlinear autoregressions}. Finally, and most importantly for our present focus, nonlinearity of the state transition function is no guarantee of long memory. We show that a class of autoregressive processes in which the state is subject to iterated nonlinear transformations still fails to achieve a slowly decaying autocovariance sequence \citep{lin, gourieroux}.

\end{itemize}

\paragraph{Semiparametric estimation of long memory.}

Methods for the estimation of the long memory parameter $d$ have been developed and analyzed under increasingly broad conditions. Here, we focus on semiparametric methods, which offer consistent estimation of the long memory without the need to estimate or even specify a full parametric model. The term ``semiparametric" refers to the fact that the estimation problem involves a finite-dimensional parameter of interest (the long memory vector) and an infinite-dimensional nuisance parameter (the spectral density).

Semiparametric estimation in the Fourier domain leverages the implication of Eq. \eqref{eq:lm_sdf} that 
\begin{align} f_X(\lambda) &\sim c_f \lvert \lambda \rvert^{-2d} \end{align}
as $\lambda \to 0$, with $c_f$ a nonzero constant. Estimators are constructed directly from the periodogram using only terms corresponding to frequencies near the origin. The long memory parameter $d$ is estimated either by log-periodogram regression, which yields the Geweke-Porter-Hudak (GPH) estimator \citep{geweke}, or through a local Gaussian approximation, which gives the Gaussian semiparametric estimator (GSE) \citep{robinson}. The GSE offers greater efficiency, requires weaker distributional assumptions, and can be defined for both univariate and multivariate time series; therefore it will be our main focus.

\paragraph{Multivariate long memory processes.}  

Analysis of long memory in multivariate stochastic processes is a topic of more recent investigation in the time series literature. The common underlying assumption in multivariate semiparametric estimation of long memory is that the real, vector-valued process $X_t \in \mathbb{R}^p$, can be written as
\begin{align} \begin{bmatrix} (1-B)^{d_1} & \ \ & 0 \\ \ \ & \ddots &  \ \ \\ 0 & \ \ & (1-B)^{d_p}  \end{bmatrix} \begin{bmatrix} X_{t1} \\ \vdots \\ X_{tp}  \end{bmatrix} &= \begin{bmatrix} U_{t1}  \\ \vdots \\ U_{tp}  \end{bmatrix}, \label{eq:multivar} \end{align}
where $X_{ti}$ is the $i^{th}$ component of $X_t$, $U_t \in \mathbb{R}^p$ is a second-order stationary process with spectral density function bounded and bounded away from zero at zero frequency, $B$ is the backshift in time operator, and $|d_i| < 1/2$ for every $i = 1,...,p$ \citep{shimotsu}. The backshift operation $B^j X_{t} = X_{t-j}$, $j \in \mathbb{Z}$ is extended to non-integer orders via
$$ (1-B)^{-d} = \sum_{k=0}^\infty \frac{\Gamma(d+k)}{k!\Gamma(d)} B^k,$$
and thus $X_t$ is referred to as a \emph{fractionally integrated} process when $d \neq 0$. Fractionally integrated processes are the most commonly used models for data with long-range dependencies, encompassing parametric classes such as the vector autoregressive fractionally integrated moving average (VARFIMA), a multivariate and long memory extension of the popular ARMA family of time series models.

If $X_t$ is defined as in Eq. \eqref{eq:multivar}, then its spectral density function $f_X(\lambda)$ satisfies \citep{hannan}
$$ f_X(\lambda) = \Phi(\lambda,d) f_U(\lambda) \Phi^*(\lambda,d),$$
where $x^*$ denotes the complex conjugate of $x$, $f_U(\lambda)$ is the spectral density function of $U_t$ at frequency $\lambda$, and
 \begin{align*} 
 \Phi(\lambda,d) &= \text{diag}\left( (1-e^{i\lambda})^{-d_i}\right)_{i = 1,...,p}. \end{align*}
 
Given an observed sequence $(x_1,...,x_T) = x_{1:T}$ with discrete Fourier transform
$$ y_j = \frac{1}{\sqrt{2\pi T}} \sum_{t = 1}^T x_t  e^{- i \lambda_j t}, \ \ \lambda_j = 2\pi j / T,$$
the spectral density matrix is estimated at Fourier frequency $\lambda_j$ by the periodogram
$$ I(\lambda_j) = y_j y_j^*.$$

Under the assumption that $f_U(\lambda) \sim G$ as $\lambda \to 0$ for some real, symmetric, positive definite $G \in \mathbb{R}^{p\times p}$, the local behavior of $f_X(\lambda)$ around the origin is governed only by $d$ and $G$:
\begin{align}
f(\lambda_j) &\sim  \Phi(\lambda,d) G  \Phi^*(\lambda,d) \label{eq:multi_sdf}. \end{align}
 
\paragraph{The Gaussian semiparametric estimator}

The Gaussian semiparametric estimator of $d$ \citep{shimotsu} is computed from a local, frequency-domain approximation to the Gaussian likelihood based on Eq. \eqref{eq:multi_sdf}. The approximation is valid under restriction of the likelihood to a range of frequencies close to the origin. Using the identity $1-e^{-i\lambda} = 2\sin(\lambda / 2)e^{i(\pi-\lambda)/2}$, we have the approximation
\begin{align*} \Phi(\lambda,d) &\approx   \text{diag}(\lambda^{-d}e^{i(\pi-\lambda)/2}) \triangleq \Lambda(d), \end{align*}
which is valid up to an error term of order $O(\lambda^2)$.

The Gaussian log-likelihood is written in the frequency domain as \citep{whittle}
\begin{align*}
\mathcal{L}_m &= \frac{1}{m} \sum_{j=1}^m \log \det f_X(\lambda_j) + \Tr \left[  f_X(\lambda_j)^{-1} y_jy_j^*\right] \\
&\approx \frac{1}{m} \sum_{j=1}^m \Big[ \log \det \Lambda_j(d) G \Lambda^*_j(d) + \Tr  \left[ \left(\Lambda_j(d) G \Lambda^*_j(d) \right)^{-1} I(\lambda_j) \right]  \Big]. \end{align*}
Validity of the approximation is ensured by restriction of the sum to the first $m$ Fourier frequencies, with $m = o(T)$.

Solving the first-order optimality condition
\begin{align*} 
\frac{\partial \mathcal{L}_m}{\partial G} = \frac{1}{m} \sum_{j=1}^m &\Big[ (G^T)^{-1} - \left( G^{-1}\Lambda_j(d)^{-1} I(\lambda_j)\Lambda^*_j(d)^{-1}G^{-1}\right)^T  \Big] = 0 \end{align*}
for $G$ yields 
\[ \widehat{G}(d) = \frac{1}{m} \sum_{j=1}^m \text{Re}\left[\Lambda_j(d)^{-1} I(\lambda_j) \Lambda_j^*(d)^{-1} \right] .\]

Substitution back into the objective results in the expression
\begin{align} 
\mathcal{L}_m(d) = \log \det \widehat{G}(d) - 2 \sum_{i = 1}^p d_i \sum_{j = 1}^m \log \lambda_j,  \label{eq:loglik} \end{align}
and the Gaussian semiparametric estimator is obtained as the minimizer
\begin{align} 
\hat{d}_{\text{GSE}} &= \argmin_{d \in \Theta} \mathcal{L}_m(d), \label{eq:dGSE} \end{align}
over the feasible set $\Theta = (-1/2,1/2)^p$.

A key result due to \citet{shimotsu} establishes that the estimator $\hat{d}_{\text{GSE}}$ is consistent and asymptotically normal under mild conditions, with
\begin{align} 
\sqrt{m} (\hat{d}_{\text{GSE}} - d_0) \to_d \mathcal{N}(0, \Omega^{-1}), \label{eq:asy} \end{align}
where
\[\Omega = 2\left[ I_p + G \odot G^{-1} + \frac{\pi^2}{4} (G \odot G^{-1} - I_p) \right], \]
 $d_0$ is the true long memory, and $\odot$ denotes the Hadamard product.

\paragraph{Optimization.}

Relatively little discussion of optimization procedures for problem in Eq. \eqref{eq:dGSE} is available in the time series literature. We are not aware of any proof that the objective is convex in the multivariate setting for instance.

To compute the estimator $\hat{d}_{\text{GSE}}$, we apply L-BFGS-B, a quasi-Newton algorithm that handles box constraints \citep{byrd}. L-BFGS-B is an iterative algorithm requiring the gradient of the objective; this is derived in Appendix C of the Supplement. 

\paragraph{Bandwidth selection} The choice of the bandwidth parameter $m$ determines the tradeoff between bias and variance in the estimator: at small $m$ the variance may be high due to few data points, while setting $m$ too large can introduce bias by accounting for the behavior of the spectral density function away from the origin. 

When it is possible to simulate from the target process, as will be the case when we evaluate criteria for long memory in recurrent neural networks, we can naturally control the variance simply by simulating long sequences and computing a dense estimate of the periodogram. Without knowledge of the shape of the spectral density function, however, it is difficult to know how to set the bandwidth to avoid bias, and thus we prefer the relatively conservative setting of $m = \sqrt{T}$. This choice is justified by a bias study for the bandwidth parameter, which is given in Appendix D of the Supplement.

	\section{Methods}
	\paragraph{RNN hidden state as a nonlinear model for a long memory process.}

The standard tool for statistical modeling of multivariate long memory processes is the vector autoregressive fractionally integrated moving average (VARFIMA) model, which represents the process $X_t \in \mathbb{R}^p$ with long memory parameter $d$ as 
\[ \Phi(B) (1-B)^d X_t = \theta(B) Z_t,\]
where $Z_t$ is a white noise process and $(1-B)^d = \text{diag}((1-B)^{d_i})$, $i = 1,..., p$ \citep{lobato, sowell}. Under the standard stationarity and invertibility conditions on the matrix polynomials $\Phi(B)$ and $\Theta(B)$, respectively, the process can be represented as
\[
X_t = (1-B)^{-d} \Phi^{-1}(B) \Theta(B) Z_t,
\]
which shows that the $X_t$ has a composite representation in terms of linear ``features" of the input sequence and an explicit fractional integration step ensuring that it satisfies the definition of multivariate long memory in Eq. \eqref{eq:multivar}.

We extend this view to deep network models for sequences with long range dependencies. The key difference is that RNN models are not constrained to work with a linear representation of the data, nor do they explicitly contain a step that guarantees the long memory of $X_t$. To evaluate long memory in an RNN model, we study the stochastic process
\begin{align}
X_t &= \Psi(Z_t), \label{eq:rnn}
\end{align}
where $Z_t$ is again a white noise, and the nonlinear transformation $\Psi$ describes the RNN transformation of inputs to the hidden state. In a typical RNN model, a decision rule is learned by linear modeling of the hidden state; this framework thus aligns with a broader theoretical characterization of deep learning as approximate linearization of complex decision boundaries in input space by means of a learned nonlinear feature representation \citep{bruna, ckn1, ckn2, bietti}.

\paragraph{Testable criteria for RNN capture of long-range dependence.}

The complexity of $\Psi(\cdot)$ corresponding to even the most basic RNN sequence models precludes a fully theoretical treatment of long memory in processes described by Eq. \eqref{eq:rnn}. Nonetheless, this characterization suggests an approach for the statistical evaluation of long memory in RNNs, as it establishes testable criteria under which a model of the form Eq. \eqref{eq:rnn} describes a process $X_t$ with long memory. In particular, to satisfy the definition in Eq. \eqref{eq:multivar} we must have
\[
X_t = \Psi(Z_t) =  (1-B)^{-d} \tilde{\Psi}(Z_t)
\]
for some $d \neq 0$ and process $\tilde{\Psi}(Z_t)$ with bounded and nonzero spectral density at zero frequency. Semiparametric estimation of $d$ in the frequency domain provides a means to evaluate this condition such that the results are agnostic to the behavior of $\tilde{\Psi}(Z_t)$ at higher frequencies. If $\Psi(Z_t)$ admits a representation in terms of an explicit fractional integration step, then this can be investigated in two complementary experiments: 

\begin{enumerate}

\item {\bf Integration of fractionally differenced input.} Define 
$$\tilde{X}_t = (1-B)^{d} Z_t, $$
where $Z_t$ is a standard Gaussian white noise and $d$ is the long memory parameter corresponding to the source $X_t$ on which the model was trained. If the sequence $\tilde{x}_{1:T}$ is drawn from $\tilde{X}_t$, then we expect to find that
$$ \hat{d}_{\text{GSE}}(\tilde{h}_{1:T}) \approx 0,$$
where $\tilde{h}_{1:T} = \Psi(\tilde{x}_{1:T})$ is the RNN hidden representation of the simulated input. On the other hand, nonzero long memory in the hidden state indicates a mismatch between fractional integration learned by the RNN and long memory of the data $X_t$.

\item {\bf Long memory transformation of white noise.} Conversely, we expect to find that the RNN hidden representation of a white noise sequence has a nonzero long memory parameter. White noise has a constant spectrum and thus a long memory parameter equal to zero. If $\Psi(\cdot)$ performs both the feature representation and fractional integration functions that are handled separately and explicitly in the VARFIMA model, then a zero-memory input will be transformed to a nonzero-memory sequence of hidden states.

\end{enumerate}

\paragraph{Total memory.}

It is common for sequence embeddings and RNN hidden layers to have hundreds of dimensions, and thus long memory estimation for these sequences naturally occurs in a high-dimensional setting. This topic is virtually unexplored in the time series literature, where multivariate studies tend to have modest dimension. Practically, this raises two main issues. First, if $p \approx m$ for dimension $p$ and bandwidth $m$, then the approximation of the test statistic distribution by its asymptotic limit will be of poor quality, and the resulting test is likely to be miscalibrated. Second, it becomes difficult to interpret the long memory vector $d$, particularly when the coordinates of the corresponding time series are not meaningful themselves.

We resolve both issues by considering the \emph{total memory} statistic $\bar{d}$, defined as
\begin{align} \bar{d} &= \mathbbm{1}^T\hat{d}_{\text{GSE}}. \label{eq:tm} \end{align}

Computation of the total memory is no more complex than that of the GSE, and it has an intuitive interpretation as the coordinate-wise aggregate strength of long memory in a multivariate time series. 

\paragraph{Asymptotic normality of the total memory estimator.}

The total memory is a simple linear functional of the GSE, and thus its consistency and asymptotic normality can be established by a simple argument. In particular, defining 
$$ \bar{d} = g(d) \triangleq \mathbbm{1}^T\hat{d}_{\text{GSE}},$$
we see that $\nabla g(d) = \mathbbm{1}$, which is clearly nonzero at zero, so that by Eq. \eqref{eq:asy} and the delta method we have
\begin{align}
\sqrt{m} (\bar{d} - \bar{d}_0) &\to_d \mathcal{N}(0, \mathbbm{1}^T\Omega^{-1} \mathbbm{1}), \label{eq:tm_asy} \end{align}
where $\bar{d}_0$ is the true total memory of the observed process.

To validate this proposed estimator, we provide a ``sanity check" on simulated high-dimensional data with known long memory in Appendix E of the Supplement. 

\paragraph{Visualizing and testing for long memory in high dimensions.}

The visual time-domain summary of long memory in Figure \ref{fig:longmem_ex} can be extended to the multivariate setting. In this case, the autocovariance $\gamma(k) = \Cov(X_t,X_{t+k})$ is matrix-valued, which for the purpose of evaluating long memory can be summarized by the scalar $\Tr( \vert \gamma(k) \vert)$, where the absolute value is taken element-wise. Recall that a sufficient condition for short memory is the absolute convergence of the autocovariance series, whereas this series diverges for long memory processes.

From a testing perspective, a statistical decision rule for the presence of long memory can be derived from the asymptotic distribution of the corresponding estimator. However, when the dimension $p$ is large and we conservatively set the bandwidth $m = \sqrt{T}$, we may have $m\approx p$ even when the observed sequence is relatively long.

The classical approach to testing for the multivariate Gaussian mean is based on the Wald statistic
$$ m(\hat{d}-d_0)^T\Omega(\hat{d}-d_0),$$
which has a $\chi^2(p)$ distribution under the null hypothesis $\mathcal{H}_0: d = d_0$. 

In Appendix F of the Supplement, we give a demonstration that the standard Wald test can be seriously miscalibrated when $m\approx p$, whereas testing for long memory with the total memory statistic remains well-calibrated in this setting. These results are consistent with previous observations that the Wald test for long memory can have poor finite-sample performance even in low dimensions \citep{shimotsu,hurvich}, though these studies suggest no alternative.

	\section{Experiments}
	\subsection{Long memory in language and music}
Much of the development of deep recurrent neural networks has been motivated by the goal of finding good representations and models for text and audio data. Our results in this section confirm that such data can be considered as realizations of long memory processes.\footnote{Code for all results in this section is available at \url{https://github.com/alecgt/RNN_long_memory}} A full summary of results is given in Table \ref{tab:data_res}, and autocovariance partial sums are plotted in Figure \ref{fig:data_longmem}. To facilitate comparison of the estimated long memory across time series of different dimension, we report the normalized total memory $\bar{d}/p = (\mathbbm{1}^T\hat{d}_{\text{GSE}})/p$ in all tables. 

For all experiments, we test the null hypothesis
$$\mathcal{H}_0: \bar{d_0} = 0$$
against the one-sided alternative of long memory,
$$\mathcal{H}_1: \bar{d_0} > 0.$$
We set the level of the test to be $\alpha = 0.05$ and compute the corresponding critical value $c_\alpha$ from the asymptotic distribution of the total memory estimator. Given an estimate of the total memory $\bar{d}(x_{1:T})$, a p-value is computed as $P(\bar{d} > \bar{d}(x_{1:T}) | \bar{d_0} = 0)$; note that a p-value less than $\alpha = 0.05$ corresponds to rejection of the null hypothesis in favor of the long memory alternative.

{\renewcommand{\arraystretch}{1.3}
\begin{table}[h]
\small
\captionsetup{size=small}
\setlength{\tabcolsep}{3pt} 
\caption{Total Memory in Natural Language and Music Data.}
\begin{tabularx}{\columnwidth}{@{} c *{1}{C} *{1}{C} *{1}{C} c @{}}
    \toprule
     & {\bf Data} & {\bf Norm. total memory} & {\bf p-value} & {\bf Reject $\mathcal{H}_0?$} \\
    \cmidrule(l){2-5}
    
    \multirow{ 4}{*}{\shortstack[l]{Natural \\ language}} & Penn TreeBank & 0.163 & $<$1 $\times 10^{-16}$ & \checkmark \\
    & Facebook CBT & 0.0636 & $<$1 $\times 10^{-16}$ & \checkmark \\
    & King James Bible & 0.192 & $<$1 $\times 10^{-16}$ & \checkmark \\  \cmidrule(l){2-5}
    
    \multirow{ 4}{*}{Music} & J.S. Bach & 0.0997 & $<$1 $\times 10^{-16}$ & \checkmark \\
    & Miles Davis & 0.322 & $<$1 $\times 10^{-16}$ & \checkmark \\
    & Oum Kalthoum & 0.343 & $<$1 $\times 10^{-16}$ & \checkmark \\
    \bottomrule
\end{tabularx}
\label{tab:data_res}
\end{table}
}

\begin{figure*}
  \includegraphics[width=\textwidth,height=5cm]{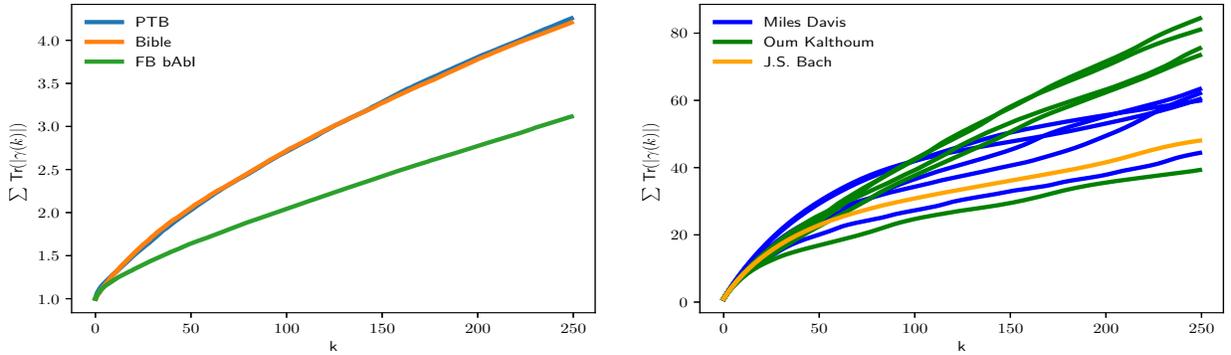}
  \caption{\small Partial sum of the autocovariance trace for embedded natural language and music data. \emph{Left}: Natural language data. For clarity we include only the longest of the 98 books in the Facebook bAbI training set. \emph{Right:} Music data. Each of the five tracks from both Miles Davis and Oum Kalthoum is plotted separately, while the Bach cello suite is treated as a single sequence.}
   \label{fig:data_longmem}
\end{figure*}

\paragraph{Natural language data.} We evaluate long memory in three different sources of English language text data: the Penn TreeBank training corpus \citep{marcus}, the training set of the Children's Book Test from Facebook's bAbI tasks \citep{weston}, and the King James Bible. The Penn TreeBank corpus and King James Bible are considered as single sequences, while the Children's Book Test data consists of 98 books, which are considered as separate sequences. We require that each sequence be of length at least $T = 2^{14}$, which ensures that the periodogram can be estimated with reasonable density near the origin. Finally, we use GloVe embeddings \citep{pennington} to convert each sequence of word tokens to an equal-length sequence of real vectors of dimension $k = 200$. 

The results show significant long memory in each of the text sources, despite their apparent differences. As might be expected, the children's book measured from the Facebook bAbI dataset demonstrates the weakest long-range dependencies, as is evident both from the value of the total memory statistic and the slope of the autocovariance partial sum.

\paragraph{Music data.}  Modeling and generation of music has recently gained significant visibility in the deep learning community as a challenging set of tasks involving sequence data. As in the natural language experiments, we seek to evaluate long memory in a broad selection of representative data. To this end, we select a complete Bach cello suite consisting of 6 pieces from the MusicNet dataset \citep{thickstun1}, the jazz recordings from Miles Davis' \emph{Kind of Blue}, and a collection of the most popular works of famous Egyptian singer Oum Kalthoum.

For the Bach cello suite, we embed the data from its raw scalar wav file format using a reduced version of a deep convolutional model that has recently achieved near state-of-the-art prediction accuracy on the MusicNet collection of classical music \citep{thickstun2}. Details of the model training, including performance benchmarks, are provided in Appendix H of the Supplement.

We are not aware of a prominent deep learning model for either jazz music or vocal performances. Therefore, for the recordings of Miles Davis and Oum Kalthoum, we revert to a standard method and extract mel-frequency cepstral coefficients (MFCC) from the raw wav files at a sample rate of $32000$ Hz \citep{logan}. A study of the impact of embedding choice on estimated long memory, including a long memory analysis of the Bach data under MFCC features, is provided in Appendix G.

The results show that long memory appears to be even more strongly represented in music than in text. We find evidence of particularly strong long-range dependence in the recordings of Miles Davis and Oum Kalthoum, consistent with their reputation for repetition and self-reference in their music. 

Overall, while the results of this section are unlikely to surprise practitioners familiar with the modeling of language and music data, they are scientifically useful for two main reasons: first, they show that our long memory analysis is able to identify well-known instances of long-range dependence in real-world data; second, they establish quantitative criteria for the successful representation of this dependency structure by RNNs trained on such data.

\subsection{Long memory analysis of language model RNNs}

We now turn to the question of whether RNNs trained on one of the datasets evaluated above are able to represent the long-range dependencies that we know to be present. We evaluate the criteria for long memory on three different RNN architectures: long short-term memory (LSTM) \citep{hochreiter}, memory cells \citep{levy}, and structurally constrained recurrent networks (SCRN) \citep{mikolov}. Each network is trained on the Penn TreeBank corpus as part of a language model that includes a learned word embedding and linear decoder of the hidden states; the architecture is identical to the ``small" LSTM model in \citep{zaremba}, which is preferred for the tractable dimension of the hidden state. Note that our objective is not to achieve state-of-the-art results, but rather to reproduce benchmark performance in a well-known deep learning task. Finally, for comparison, we will also include an untrained LSTM in our experiments; the parameters of this model are simply set by random initialization.

{\renewcommand{\arraystretch}{1.3}
\begin{table}[h!]
\small 
\captionsetup{size=small}
\setlength{\tabcolsep}{3pt} 
\caption{Language Model Performance by RNN Type}
\begin{tabularx}{\columnwidth}{@{} *{1}{C} *{1}{C} @{}}
    \toprule
     {\bf Model} & {\bf Test Perplexity} \\
    \cmidrule(l){1-2}
    Zaremba et al. & 114.5 \\
     LSTM & 114.5 \\
     Memory cell & 119.0 \\
     SCRN &  124.3 \\
    \bottomrule
\end{tabularx}
\label{tab:lang}
\end{table}
}

\paragraph{RNN integration of fractionally differenced input.}

Having estimated the long memory parameter $d$ corresponding to the Penn TreeBank training data in the previous section, we simulate inputs $\tilde{x}_{1:T}$ with $T = 2^{16}$ from by fractional differencing of a standard Gaussian white noise and evaluate the total memory of the corresponding hidden representation $\Psi(\tilde{x}_{1:T})$ for each RNN. Results from $n=100$ trials are compiled in Table \ref{tab:rnn_pos_res} (standard error of total memory estimates in parentheses). We test the null hypothesis $\mathcal{H}_0: \bar{d} = 0$ against the one-sided alternative $\mathcal{H}_1: \bar{d} < 0$, which corresponds to the model's failure to represent the full strength of fractional integration observed in the data.

{\renewcommand{\arraystretch}{1.3}
\begin{table}[h]
\small 
\captionsetup{size=small}
\setlength{\tabcolsep}{3pt} 
\caption{Residual Total Memory in RNN Representations of Fractionally Differenced Input.}
\begin{tabularx}{\columnwidth}{@{} c *{1}{C} *{1}{C} c @{}}
    \toprule
    {\bf Model} & {\bf Norm. total memory} & {\bf p-value} & {\bf Reject $\mathcal{H}_0?$} \\
    \cmidrule(l){1-4}
    
     LSTM (trained) & $-8.36 \times 10^{-3}$ (0.00475) &$4.07 \times 10^{-2}$ & \checkmark \\
     LSTM (untrained) &$-6.20 \times 10^{-2}$ (0.00387) & $<$1 $\times 10^{-16}$ & \checkmark \\
     Memory cell & $-1.18 \times 10^{-2}$ (0.00539) & $1.52 \times 10^{-2}$ & \checkmark \\
     SCRN & $-2.62 \times 10^{-2}$ (0.00631) &  $3.32 \times 10^{-5}$ & \checkmark \\  
    \bottomrule
\end{tabularx}
\label{tab:rnn_pos_res}
\end{table}
}

%

\paragraph{RNN transformation of white noise.}

For a complementary analysis, we evaluate whether the RNNs can impart nontrivial long-range dependency structure to white noise inputs. In this case, the input sequence $z_{1:T}$ is drawn from a standard Gaussian white noise process, and we test the corresponding hidden representation $\Psi(z_{1:T})$ for nonzero total memory. As in the previous experiment, we select $T = 2^{16}$, choose the bandwidth parameter $m = \sqrt{T}$, and simulate $n=100$ trials for each RNN. Results are detailed in Table \ref{tab:rnn_res}. We test $\mathcal{H}_0: \bar{d}_0 = 0$ against $\mathcal{H}_1: \bar{d}_0 > 0$; here, the alternative corresponds to successful transformation of white noise input to long memory hidden state.

{\renewcommand{\arraystretch}{1.3}
\begin{table}[h]
\small 
\captionsetup{size=small}
\setlength{\tabcolsep}{3pt} 
\caption{Total Memory in RNN Representations of White Noise Input.}
\begin{tabularx}{\columnwidth}{@{} c *{1}{C} *{1}{C} c @{}}
    \toprule
    {\bf Model} & {\bf Norm. total memory} & {\bf p-value} & {\bf Reject $\mathcal{H}_0?$} \\
    \cmidrule(l){1-4}

     LSTM (trained) & $-8.59 \times 10^{-4}$ (0.00405) & 0.583 & \text{\sffamily X} \\
     LSTM (untrained) &$-4.17 \times 10^{-4}$ (0.00223) & 0.572 & \text{\sffamily X} \\
     Memory cell & $-5.96 \times 10^{-4}$ (0.00452) & 0.552 & \text{\sffamily X}\\
     SCRN & $2.37 \times 10^{-3}$ (0.00522) &  0.324 & \text{\sffamily X} \\  
    \bottomrule
\end{tabularx}
\label{tab:rnn_res}
\end{table}
}

\paragraph{Discussion.} We summarize the main experimental result as follows: there is a statistically well-defined and practically identifiable property, relevant for prediction and broadly represented in language and music data, that is not present according to two fractional integration criteria in a collection of RNNs trained to benchmark performance.

Tables \ref{tab:rnn_pos_res} and \ref{tab:rnn_res} show that each evaluated RNN fails both criteria for representation of the long-range dependency structure of the data on which it was trained. The result holds despite a training protocol that reproduces benchmark performance, and for RNN architectures specifically engineered to alleviate the gradient issues typically implicated in the learning of long-range dependencies. 

%

	\section{Conclusion}
	We have introduced and demonstrated a framework for the evaluation of long memory in RNNs that proceeds from a well-known definition in the time series literature. Under this definition, long memory is the condition enabling meaningful autocovariance at long lags in a multivariate time series. Of course, for sufficiently complex processes, this will not fully characterize the long-range dependence structure of the data generating process. Nonetheless, it represents a practical and informative foundation upon which to develop a statistical toolkit for estimation, inference, and hypothesis testing, which goes beyond the current paradigm of heuristic checks.

The long memory framework makes possible a formal investigation of specific and quantitative hypotheses concerning the fundamental issue of long-range dependencies in deep sequence learning. The experiments presented in this work investigate this phenomenon in natural language and music data, and in the learned representations of RNNs themselves, using the total memory statistic as an interpretable quantity that avoids the challenges associated with high-dimensional testing. The results identify long memory as a broadly prevalent feature of natural language and music data, while showing evidence that benchmark recurrent neural network models designed to capture this phenomenon may in fact fail to do so. Finally, this work suggests future topics in both time series, particularly concerning long memory analysis in high dimensions, and in deep learning, as a challenge to learn long memory representations in RNNs.

	\section*{Acknowledgments}

This work was supported by the Big Data for Genomics and Neuroscience Training Grant 8T32LM012419, NSF TRIPODS Award CCF-1740551, the program ``Learning in Machines and Brains" of CIFAR, and faculty research awards.

\bibliography{longmem_refs_cleaned}
\bibliographystyle{abbrvnat} 

\onecolumn
\section*{Appendix}
\appendix
\section{Slowly varying functions}

\begin{definition}[Slowly varying at infinity / zero] A positive function $L$ is said to be slowly varying at infinity if for any $u > 0$,
$$ L(xu) \sim L(u) \ \ \text{ as } x\to\infty.$$
A function $L$ is slowly varying at zero if $L(1/u)$ is slowly varying at infinity. \end{definition}

Slowly varying functions serve an important purpose in the study of long memory processes by significantly expanding the class of autocovariance (for slowly varying at infinity) or spectral density (slowly varying at zero) functions described without changing their relevant asymptotic behavior. In particular, if we define 
$$ R(u) = u^\rho L(u),$$
with $\rho \in \mathbb{R}$, then $R(au)/R(u) \to a^\rho$ as $u \to \infty$, so that the slowly varying function can be ignored in the limit. Properties of slowly varying functions are also used to show equivalence between various notions of long memory, or establish conditions under which such notions are equivalent. For example, the following proposition relates the time-domain definition of long memory to the summability of the autocovariance function:

\begin{proposition}[Prop. 2.2.1, \cite{pipiras}] Let $L$ be slowly varying at infinity and $p>-1$. Then 
$$ \gamma_k = L(k) k^p, \ \ k\geq 1$$
implies that
$$ \sum_{k=1}^n \gamma_k \sim \frac{L(n)n^{p+1}}{p+1}, \ \ \text{ as } n\to\infty.$$
\end{proposition}

The equivalence between the time and spectral domain definitions of long memory can be established under the condition that the slowly varying part is quasi-monotone.

\begin{definition}[Quasi-monotone] A slowly varying function $L$ on $[0,\infty)$ quasi-monotone if it is of bounded variation on any compact interval and if there exists some $\delta >0$ such that
$$ \int_0^x u^\delta \vert dL(u) \vert = O(x^\delta L(x)), \ \ \text{ as } x\to\infty.$$ \end{definition}

For a proof we refer to \cite{pipiras} Section 2.2.4. A thorough treatment of slowly varying functions in the context of long memory is available in \cite{pipiras} Sections 2.1-2.2 and \cite{beran} Section 1.3.
\section{Short memory of common time series models}

We first note that in order to show that the class of processes described by a given time series model has short memory, it is sufficient to show 
\begin{align} 
\sum_{k=-\infty}^\infty \vert \gamma_X(k) \vert < \infty \label{crit} \end{align}
for each process $X_t$ belonging to the parametric family. The property $\sum_k \vert \gamma_X(k) \vert = \infty$ is implied by both the frequency and time domain definitions of long memory for a scalar process, which are themselves equivalent under the condition that the slowly varying part of the spectral density near zero is quasimonotone. Therefore, establishing \eqref{crit} for a given class of models implies that they do not satisfy the definition of a long memory process.

\begin{proposition}
Let $X_t$ be an irreducible and aperiodic Markov chain on a finite state space $\mathcal{X}$ such that its corresponding transition matrix $P$ has distinct eigenvalues. Let $g: \mathcal{X} \to \mathbb{R}$, and define $Y_t = g(X_t)$. Then $Y_t$ is a short memory process. 
\end{proposition}
\begin{proof}
Computation of the autocovariance for a finite state Markov model is classical, but we include it here for completeness. Let $X_t$ be an irreducible and aperiodic Markov chain on the finite space $\mathcal{X} = \{1,...,m\}$, and suppose that the transition matrix $P$ (where $P_{ij} = p(X_{t+1} = i | X_t = j)$) has distinct eigenvalues. Then $X_t$ has a unique stationary distribution, and we denote its elements $p(X_t = i) = \pi_i$. 

Let $X_0$ have the distribution $\pi$, and define $Y_t = g(X_t)$ for $t \geq 0$ and some $g: \mathcal{X} \to \mathbb{R}$. Note that $Y_t$ is stationary since $X_t$ is stationary. We will show that the scalar process $Y_t$ has short memory. 

Write the autocovariance
\begin{align*}
\gamma_Y(k) &= \mathbb{E} Y_k Y_0 - [\mathbb{E} Y_0 ]^2 \\
&= \sum_{i=1}^m \sum_{j=1}^m g(i) g(j) p_{ij}^{(k)}\pi_j  - \sum_{i=1}^m \sum_{j=1}^m g(i) g(j) \pi_i \pi_j \\ 
&= \sum_{i=1}^m \sum_{j=1}^m g(i) g(j) \pi_j \left( p_{ij}^{(k)} - \pi_i \right), \end{align*}
where $p_{ij}^{(k)} = p(X_{t+k} = i \vert X_t = j)$.

Since $P$ has distinct eigenvalues, it is similar to a diagonal matrix $\Lambda$:
$$ P = Q\Lambda Q^{-1},$$
so that 
$$ P^k = Q\Lambda^k Q^{-1} = \sum_{i=1}^m \lambda_i^k q_i \tilde{q}_i',$$ 
where $q_i$ and $\tilde{q}_i$ denote the $i^{th}$ row of $Q$ and $Q^{-1}$, respectively, and $x'$ denotes the transpose of $x$. Furthermore, from the existence of the unique stationary distribution $\pi = (\pi_1 ,..., \pi_m)'$ we have that 
$$ P \pi = \pi$$
so that $\lambda_1 = 1$, and since $P$ is a stochastic matrix, the corresponding left eigenvalue is
$$ \begin{bmatrix} 1 & ... & 1 \end{bmatrix} P  = \begin{bmatrix} 1 & ... & 1 \end{bmatrix} .$$

Thus
$$ \lambda_1 q_1 \tilde{q}_i' = \pi \begin{bmatrix} 1 & ... & 1 \end{bmatrix} = \begin{bmatrix} \pi_1 & ... & \pi_1 \\ \vdots & \ddots & \vdots \\ \pi_m & ... & \pi_m \end{bmatrix} = \Pi.$$

Then we can write
$$ P^k  = \Pi + \sum_{i=2}^m \lambda_i^k q_i \tilde{q}_i',$$
and since the $\lambda_i$'s are distinct, $| \lambda_i | < 1$ for $i = 2,...,m$. Therefore,
$$ | p_{ij}^{(k)} - \pi_i | < C_1 s^k$$
for some $s \in (0,1)$, which from above implies that 
$$ \gamma_Y(k) < C_2 s^k.$$	

The absolute convergence of the autocovariance series then follows by comparison to the dominating geometric series.
\end{proof}

Furthermore, as we next show, neither extension of the Markov chain to higher (finite) order or taking (finite) mixtures of Markov chains is sufficient to obtain a long memory process. We provide a novel proof that the mixture transition distribution (MTD) model \citep{raftery} for high-order Markov chains defines a short memory process under conditions similar to those of the proof above.

\begin{proposition}
Let $X_t$ be an order-$p$ Markov chain whose transition tensor is parameterized by the MTD model 
\begin{align} p(X_t = i | X_{t-1} = j_1, ..., X_{t-p} = j_p) = \sum_{\ell = 1}^p \lambda_\ell Q^{(\ell)}_{ij_\ell}, \label{mtd} \end{align}
where each $Q^{(\ell)}$ is a column-stochastic matrix, $\lambda_\ell >0$ for each $\ell = 1,...,p$, and $\sum_\ell \lambda_\ell = 1$. Suppose that the state space $\mathcal{X}$ is finite with $|\mathcal{X}| = m$, and we define $Y_t = g(X_t)$ for some $g: \mathcal{X} \to \mathbb{R}$. Then $Y_t$ is a short memory process.
\end{proposition}
\begin{proof}
In order to write the autocovariance sequence of an MTD process, we must first establish its stationary distribution. Let $Q \in \mathbb{R}^{m^p \times m^p}$ denote the multivariate Markov transition matrix, which has entries
\begin{align*} 
q_{s_{0:p-1}, s'_{1:p}} = p(X_{t} &= s_{0}, ..., X_{t-p+1} = s_{p-1} | X_{t-1} = s'_{1}, ..., X_{t-p} = s'_{p}) \\
&= \begin{cases} \sum_{\ell=1}^p \lambda_\ell Q^{(\ell)}_{s_{0}s'_{\ell}} & \\ \text{ if } s_t = s'_t \text{ for } t= 1,...,p-1 &\text{ otherwise.} \end{cases} \end{align*}

We make the following assumptions on $Q$:
\begin{itemize}
\item $Q$ has distinct eigenvalues
\item Each $Q^{(\ell)}$ has strictly positive elements on the diagonal
\end{itemize}

Each state of $Q$ is reachable from all others, so $Q$ is irreducible. The second assumption above shows that the states corresponding to the $m$ nonzero diagonal elements of $Q$ are aperiodic, and thus $Q$ is aperiodic. The transition matrix $Q$ therefore specifies an ergodic Markov chain and hence has a unique stationary distribution $\xi \in \mathbb{R}^{m^p}$. We will denote by $\pi \in \mathbb{R}^m$ the univariate marginal of $\xi$.

Now let $\xi \in \mathbb{R}^{m^p}$ be the multivariate stationary distribution of $X_t$, and let $\pi \in \mathbb{R}^m$ be its univariate marginal. Let $(X_{-p},...,X_{-1})$ have the distribution $\xi$, and define $X_t$ according to \eqref{mtd} for $t = 0,1,2,...$. Then both $X_t$ and $Y_t = g(X_t)$ are stationary.

The autocovariance $\gamma_Y(t)$ can be written as 
$$ \gamma_Y(k) = \sum_{i=1}^m \sum_{j=1}^m g(i) g(j) \pi_j \left( p_{ij}^{(k)} - \pi_i \right),$$
where $p_{ij}^{(k)} = p(X_{t+k} = i \vert X_t j)$.

Observe that the transition probability $p_{ij}^{(k)}$ can be obtained from the $k$-step multivariate transition matrix $Q^k$ via
\begin{align*} 
p_{ij}^{(k)} &= p(X_{t+k} = i | X_t = j) \\  
&= \sum_{s_{1:p-1}} p(X_{t+k} = i, X_{t+k-1:t+k-p+1} = s_{1:p-1} | X_{t} = j) \\
&= \sum_{s_{1:p-1}} \sum_{s'_{1:p-1}} p(X_{t+k} = i, X_{t+k-1:t+k-p+1} = s_{1:p-1} | X_{t} = j, X_{t-1:t-p+1} = s'_{1:p-1}) p(X_{t-1:t-p+1} = s'_{1:p-1}) \\
&= \sum_{s_{1:p-1}} \sum_{s'_{1:p-1}} q_{is_{1:p-1}, js'_{1:p-1}}^{(k)} p(X_{t-1} = s'_1, ..., X_{t-p+1} = s'_{p-1}). \end{align*}

We note that the summation over $s_{1:p-1}$ is precisely the marginalization required to obtain $\pi_i$ from $\xi$. Therefore, we can write
\begin{align*}
|p_{ij}^{(k)} - \pi_i | &= \left\vert \sum_{s_{1:p-1}} \left( \left[  \sum_{s'_{1:p-1}} q_{is_{1:p-1}, js'_{1:p-1}}^{(k)} p(X_{t-1} = s'_1, ..., X_{t-p+1} = s'_{p-1}) \right] - \xi_{is_{1:p-1}}\right) \right\vert \\
&\leq  \sum_{s_{1:p-1}} \left\vert \left( \left[  \sum_{s'_{1:p-1}} q_{is_{1:p-1}, js'_{1:p-1}}^{(k)} p(X_{t-1} = s'_1, ..., X_{t-p+1} = s'_{p-1}) \right] - \xi_{is_{1:p-1}}\right) \right\vert \end{align*}

However, for each $q_{s_{0:p-1}, s'_{1:p}}$ we have 
$$ |q_{s_{0:p-1}, s'_{1:p}} - \xi_{s_{0:p-1}} | < Cs^k$$
for some $s \in (0,1)$ by an argument analogous to the Markov chain example. This implies 
$$ \left\vert \sum_{s'_{1:p-1}} q_{is_{1:p-1}, js'_{1:p-1}}^{(k)} p(X_{t-1} = s'_1, ..., X_{t-p+1} = s'_{p-1})  - \xi_{is_{1:p-1}} \right\vert < Cs^k$$
since $\sum_{s'_{1:p-1}} q_{is_{1:p-1}, js'_{1:p-1}}^{(k)} p(X_{t-1} = s'_1, ..., X_{t-p+1} = s'_{p-1})$ is a convex combination of elements obeying the same bound. Therefore, we have
\begin{align*}
|p_{ij}^{(k)} - \pi_i | &\leq  \sum_{s_{1:p-1}} \left\vert \left( \left[  \sum_{s'_{1:p-1}} q_{is_{1:p-1}, js'_{1:p-1}}^{(k)} p(X_{t-1} = s'_1, ..., X_{t-p+1} = s'_{p-1}) \right] - \xi_{is_{1:p-1}}\right) \right\vert \\
&< \sum_{s_{1:p-1}} Cs^k \\
&= \tilde{C} s^k, \end{align*}
and hence the MTD model has short memory with exponentially decaying autocovariance.

\end{proof}

For processes on a real-valued state space, the autoregressive moving average (ARMA) model is a well-known and widely used tool. ARMA models have good approximation properties, as evidenced by the existence of AR and MA orders guaranteeing aribitrarily good approximation to a stationary real-valued stochastic process with continuous spectral density \citep{davis}. Furthermore, ARMA models with nontrivial moving average components are equivalent to autoregressive models of infinite order, suggesting that these models can integrate information over long histories. However, despite these appealing properties, this class of models cannot represent statistical long memory.

\begin{proposition}
Define the ARMA process $X_t$ by
$$ \phi(B)X_t = \theta(B)Z_t,$$
where $Z_t$ is a white noise process with variance $\sigma^2$ and $\phi(z) \neq 0$ for all $z \in \mathbb{C}$ such that $|z| = 1$. Then $X_t$ is a short memory process.
\end{proposition}
\begin{proof}
As in the Markov chain case, the proof is classical but included for completeness. Let $X_t$ be defined as in the statement above. Then $X_t$ has the representation
$$ X_t = \sum_{j=-\infty}^\infty \psi_j Z_{t-j}$$
where the coefficients $\psi_j$ are given by
$$\theta(z)\phi(z)^{-1} = \psi(z) = \sum_{j=-\infty}^\infty \psi_j z^j,$$
with the above series absolutely convergent on $r^{-1} < |z| < r$ for some $r>1$ (cf. \cite{davis}, Chapter 3).

Absolute convergence implies that there exists some $\epsilon > 0$ and $L < \infty$ such that
$$  \sum_{j= -\infty}^\infty \vert \psi_j (1+\epsilon)^j \vert  = \sum_{j= -\infty}^\infty \vert \psi_j \vert (1+\epsilon)^j = L,$$
so that there exists a $K < \infty$ for which
$$ \vert \psi_j \vert < \frac{K}{(1+\epsilon)^j}.$$ 

The autocovariance can be expressed as
$$ \gamma_X(k) = \sigma^2 \sum_{j=-\infty}^\infty \psi_j \psi_{j+|k|},$$
and thus we can write
\begin{align*}
\vert \gamma_X(k) \vert &= \sigma^2 \left\vert  \sum_{j=-\infty}^\infty \psi_j \psi_{j+|k|} \right\vert \\
&\leq \sigma^2  \sum_{j=-\infty}^\infty \vert \psi_j \vert \vert \psi_{j+|k|} \vert \\
&\leq \sigma^2K^2  \sum_{j=-\infty}^\infty \frac{1}{(1+\epsilon)^j}\frac{1}{(1+\epsilon)^{j+|k|}} \\
&= \left[ \sigma^2K^2  \sum_{j=-\infty}^\infty \frac{1}{(1+\epsilon)^{2j}} \right] \frac{1}{(1+\epsilon)^{|k|}} \\
&= Cs^{|k|} \end{align*}
for $s = 1/(1+\epsilon) \in (0,1)$.

Therefore, as with the Markov models, the autocovariance sequence of an ARMA process is not only absolutely summable but also dominated by an exponentially decaying sequence.
\end{proof}

Finally, we show that in general nonlinear state transitions are not sufficient to induce long range dependence, a point particularly relevant to the analysis of long memory in RNNs.
\begin{proposition}
Define the scalar nonlinear autoregressive process
$$ X_{t+1} = f(X_t) + \varepsilon_t,$$
where $\{\varepsilon_t\}$ is a white noise sequence with positive density with respect to Lebesgue measure and satisfying $\mathbb{E}|\varepsilon_t| <\infty$, while $f: \mathbb{R} \to \mathbb{R}$ is bounded on compact sets and satisfies
$$ \sup_{|x| > r} \left| \frac{f(x)}{x} \right| < 1$$
for some $r > 0$.
Then $X_t$ has a unique stationary distribution $\pi$, and the sequence of random variables $\{X_t, t\geq 0\}$ initialized with $X_0 \sim \pi$ is strictly stationary and geometrically ergodic.

Furthermore, if 
$$ \mathbb{E}|X_{t}|^{2+\delta} < \infty$$ 
for some $\delta >0$, then $\{X_t\}$ is a short memory process.
\end{proposition}
\begin{proof}
The proof proceeds by analysis of $X_t$ as a Markov chain on a general state space $(\mathbb{R},\mathcal{B})$, where $\mathcal{B}$ is the standard Borel sigma algebra on the real line. Define the transition kernel $P(x,B) = P(X_t \in B | X_{t-1} = x)$ for any $x\in \mathbb{R}$ and $B \in \mathcal{B}$.

We first establish that $X_t$ is aperiodic. A $d$-cycle is defined by a collection of disjoint sets $\{D_i\}, 0 = 1,...,d-1$ such that
\begin{enumerate}
\item For $x \in D_i$, $P(x,D_{i+1}) = 1$, $i = 0,...,d-1 \mod d$.
\item The set $[\cup_i D_i]^C$ has measure zero.
\end{enumerate}

The period is defined as the largest $d$ for which $\{X_t\}$ has a $d$-cycle \citep{meyn}. Clearly, however, since $\varepsilon_t$ has positive density with respect to Lebesgue measure, $ p(x,D) = 1$ only if $D = \mathbb{R}$ up to null sets. Thus the period is $d=1$, so $\{X_t\}$ is aperiodic.

Strict stationarity and geometric ergodicity are established by showing that the aperiodic chain $\{X_t\}$ satisfies a strengthened version of the Tweedie criterion \citep{meyn}, which requires the existence of a measurable non-negative function $g: \mathbb{R} \to \mathbb{R}$, $\epsilon > 0$, $R>1$ and $M < \infty$ such that
\begin{align*}
R \mathbb{E}[ g(X_{t+1}) | X_t = x] &\leq g(x) - \epsilon, \ \ x \in K^c \\
\mathbb{E}[ g(X_{t+1}) \mathbbm{1}\{X_{t+1} \in K^c\} | X_t = x] &\leq M, \ \ x \in K
\end{align*}
for some set $K$ satisfying
$$ \inf_{x \in K} \sum_{n=1}^m P^n(x,B) > 0$$
Under the conditions of $f$ and $\varepsilon_t$ assumed above, this criterion is established for the process $X_t$ in \cite{tjostheim} (Thm 4.1), with $g(x) = |x|$.

Geometric ergodicity implies that the 
$$ \norm{\lambda P^n - \pi}_{TV} \leq C\rho^n,$$
with $C < \infty$, $\rho \in (0,1)$, and where $\norm{\cdot}_{TV}$ denotes the total variation distance between measures. A well-known result in the theory of Markov chains \citep{bradley} establishes that geometric ergodicity is equivalent to absolute regularity, which is parameterized by
$$ \beta(k) = \sup \frac{1}{2} \sum_{i=1}^I \sum_{j=1}^J | P(A_i \cup B_j) - P(A_i) P(B_j) |,$$
where the supremum is taken over all finite partitions $\{A_1,...,A_I\}$ and $\{B_1,...,B_J\}$ of the sigma fields $\mathcal{A} = \sigma(X_t)$ and $\mathcal{B} = \sigma(X_{t+k})$. In particular, $\beta(k)$ decays at least exponentially fast. 

Furthermore, for any two sigma fields $\mathcal{A}$ and $\mathcal{B}$ we have 
\begin{align*}
\beta(\mathcal{A},\mathcal{B})  &= \sup \frac{1}{2} \sum_{i=1}^I \sum_{j=1}^J | P(A_i \cup B_j) - P(A_i) P(B_j) | \\
&\geq \sup \frac{1}{2} | P(A \cup B) - P(A)P(B)|, \ \ A \in \mathcal{A}, B \in \mathcal{B} \\
&= 2 \alpha(\mathcal{A},\mathcal{B}), \end{align*}
so that the $\alpha$-mixing parameter is also bounded by an exponentially decaying sequence.

Finally, if $\mathbb{E}|X|^{2+\delta}$ for some $\delta > 0$, then the absolute covariance obeys (\cite{ibragimov}, Thm. 17.2.2)
$$ |\gamma(k) |  = \sigma^{-2} |\rho(k) | \leq C\alpha(k)^{\delta / (2+\delta)},$$
which completes the proof.
\end{proof}

\section{Gradient of the GSE objective}

Recall that the objective function is given by

\begin{align*} 
\mathcal{L}_m(d) &= \log \det \widehat{G}(d) - 2 \sum_{i=1}^m d_i \frac{1}{m} \sum_{j=1}^m \log \lambda_j, \end{align*}
with 
\begin{align*}
\widehat{G}(d) &= \frac{1}{m} \sum_{j=1}^m \text{Re} \left[ \Lambda_j(d)^{-1} I_{T,X}(\lambda_j) \Lambda^*_j(d)^{-1} \right] \\
\Lambda_j(d) &= \text{diag}(\lambda_j^{-d}e^{i(\pi-\lambda_j)/2}) \\
I_{T,X}(\lambda_j) &= y_j y_j^*,\ \  y_j = \frac{1}{\sqrt{2\pi T}} \sum_{t = 1}^T x_t  e^{- i \lambda_j t}, \ \ \lambda_j = 2\pi j / T. \end{align*}

The derivative with respect to the element $d_\ell$ of the long memory vector $d$, for any $\ell = 1,...,p$, is
$$ \frac{\partial}{\partial d_\ell} R(d) = \Tr\left[\widehat{G}(d)^{-1} \frac{\partial}{\partial d_\ell} \widehat{G}(d) \right] - \frac{2}{m} \sum_{j=1}^m \log \lambda_j.$$

Note that Fourier frequencies $\lambda_j$ are strictly positive for $j \geq 1$, so that $\log \lambda_j$ is well defined. 

For the term $\frac{\partial}{\partial d_\ell} \widehat{G}(d)$, note that the $(h,k)$ element of the matrix $\widehat{G}(d)$ can be written as 

$$ \frac{1}{m} \sum_{j=1}^m \text{Re}\left[ I(\lambda_j)_{h,k} \exp\left( (d_h + d_k) \log \lambda_j + \frac{i(\pi - \lambda_j)(d_h - d_k)}{2} \right) \right],$$

and therefore the derivative $\frac{\partial}{\partial d_\ell} \widehat{G}(d)$ is given by

$$\left( \frac{\partial}{\partial d_\ell} \widehat{G}(d)\right)_{h,k} = \begin{cases} \frac{1}{m} \sum_{j=1}^m \text{Re}\left[ I(\lambda_j)_{\ell, k} c_j^- \exp(c_j^+ d_k) \exp(c_j^- d_\ell) \right] & \text{ for } h = \ell, h \neq k \\  \frac{1}{m} \sum_{j=1}^m \text{Re}\left[ I(\lambda_j)_{h, \ell} c_j^+ \exp(c_j^- d_h) \exp(c_j^+ d_\ell) \right] & \text{ for } k = \ell, h \neq k \\ \frac{1}{m} \sum_{j=1}^m \text{Re}\left[ 2 I(\lambda_j)_{\ell, \ell} \log \lambda_j  \exp(2 d_\ell \log \lambda_j)\right] & \text{ for } \ell = h = k \\ 0 & \text{ otherwise} \end{cases},$$

where 
\begin{align*} 
c_j^- &= \log \lambda_j - i\left(\frac{\pi - \lambda_j}{2}\right) \\
c_j^+ &= \log \lambda_j + i\left(\frac{\pi - \lambda_j}{2}\right). \end{align*}
\section{Bias study for bandwidth parameter}

We demonstrate the potential for semiparametric estimation to incur bias when the bandwidth parameter $m$ is set too high relative to the over length $T$ of the observed sequence. The bias results from inclusion of periodogram ordinates in the long memory estimator that capture behavior in the spectral density function not local to the origin. 

\begin{figure}[h!]
  \includegraphics[width=\textwidth,height=5cm]{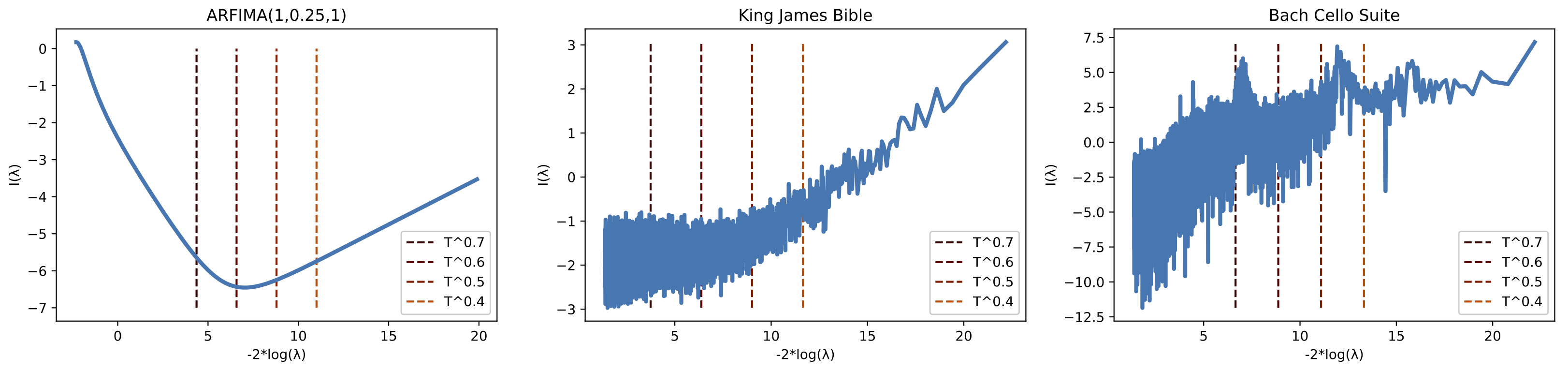}
  \caption{\small Spectral density function of an ARFIMA(1,$d$,1) process (left) and smoothed estimates of the periodogram for the first coordinate of the embedded Bible text and Bach cello suite (center and right, respectively). Cutoff points associated with four choices of the bandwidth $m$ are plotted as vertical dashed lines; the semiparametric estimate of the long memory for each sequence is essentially a measure of the slope based on the subset of points $(-2 \lambda, I(\lambda))$ to the \emph{right} of this line.}
   \label{fig:bias}
\end{figure}

We give an illustration for univariate time series, which allows us to take advantage of a convenient visual interpretation of the long memory as the slope of $\log I(\lambda_j)$ against $-2\log(\lambda_j)$ as $\lambda_j \to 0$. Figure \ref{fig:bias} shows the spectral density function corresponding to three scalar processes: an ARFIMA(1,$d$,1) process with $d = 0.25$, a univariate projection of the embedded text from the King James Bible, and a univariate projection of the embedded Bach cello suite. For the ARFIMA process, the spectral density function can be computed exactly; for the other two sequences, it is estimated by the smoothed periodogram. 

By marking the cutoff points $-2 \log \lambda_m$ associated with different choices of $m$, we indicate the subset of points $(-2 \log \lambda_j, \log I(\lambda_j))$ to the right of this cutoff used to compute the semparametric estimate of $d$. In the scalar case, this is essentially the slope of the SDF as $\lambda$ approaches zero; thus it becomes clear that bias can be introduced when points sufficiently far from the origin are included. On the other hand, choosing $m$ too small introduces the risk of high variance in the estimator; note for example that the estimate with $m = T^{0.4}$ for the Bach cello suite would be strongly influenced by a single point just to the left of $-2\log \lambda = 15$.
\section{Validation of total memory estimator}

We compute the total memory statistic
$$ \bar{d} = \mathbbm{1}^T \hat{d}_{\text{GSE}}$$
for simulated fractionally differenced Gaussian white noise sequences of dimension $k = 200$. We simulate four different settings for the long memory parameter:
\begin{itemize}

\item {\bf Zero}: Each coordinate of $d$ is equal to zero.

\item {\bf Constant}: Each coordinate of $d$ is set to the same value, $d = 0.25$.

\item {\bf Subset}: 90$\%$ of the coordinates are set to $0$, while the remaining $10\%$ are set to have strong long memory with $d = 0.4$.

\item {\bf Range}: The elements of $d$ are drawn from a scaled Beta distribution with support on $(0,0.25)$ and centered at $0.125$. 

\end{itemize}

For each setting, we simulate $n = 100$ sequences and compute the total memory. Results are plotted in Figure \ref{fig:tm_sim}, while in Table \ref{tab:tm_sim} we compare the sample mean and variance of the estimator compared to the asymptotic value stated in the main paper.

\begin{figure}[h!]
  \includegraphics[width=\textwidth,height=5cm]{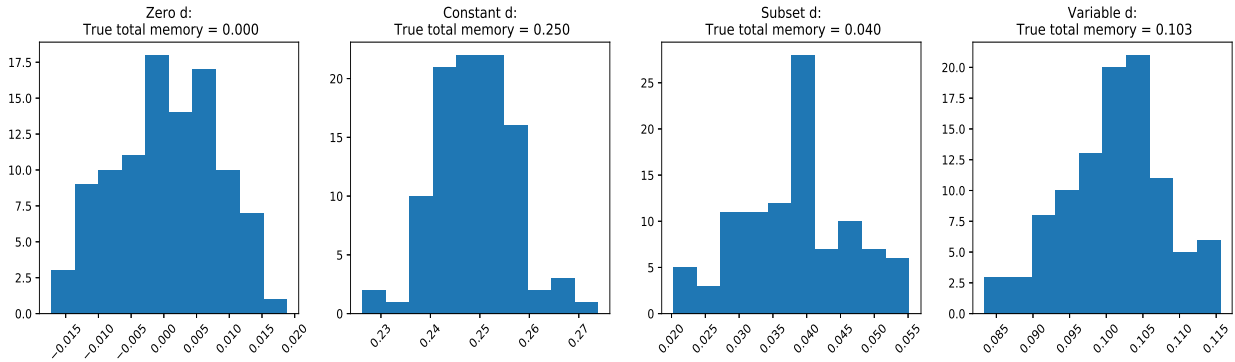}
  \caption{\small Sample distribution of the total memory estimator $\bar{d}$ in four different simulation settings.}
  \label{fig:tm_sim}
\end{figure}

{\renewcommand{\arraystretch}{1.3}
\begin{table}[h]
\centering
\small 
\captionsetup{size=small}
\setlength{\tabcolsep}{5pt} 
\caption{Comparison of the sample mean and variance for the total memory estimator with the true total memory of the generating process and the asymptotic variance of the total memory estimator (both given in parentheses).}
\begin{tabularx}{0.6\textwidth}{c *{1}{C} c}
    \toprule
    {\bf Setting} & {\bf Mean} & {\bf Variance} \\
    \cmidrule(l){1-3}
    
     Zero & $2.82 \times 10^{-4}$ (0.0) & 0.00801 (0.00698)  \\
     Constant & 0.249 (0.25) & 0.00793 (0.00698)\\
     Subset & 0.382 (0.04) & 0.00804 (0.00698)\\
     Range & 0.101 (0.1029) & 0.00696  (0.00698)\\  
    \bottomrule
\end{tabularx}
\label{tab:tm_sim}
\end{table}
}

In each of these four diverse simulation settings, the total memory estimator accurately recovers the true underlying parameter of the data generating process.
\section{Calibration of total memory vs. Wald test in high dimensions}

Here we demonstrate that the standard Wald test can be badly miscalibrated in the high-dimensional regime, whereas testing for long memory with the total memory statistic remains well-calibrated. Recall that, given an estimate $\hat{d}_{\text{GSE}}$ of the multivariate long memory parameter, the Wald statistic for the null hypothesis $\mathcal{H}_0: d = 0$ is computed as 
$$ t_{\text{Wald}} = \hat{d}_{\text{GSE}}^T(\Omega/m)\hat{d}_{\text{GSE}}.$$
This quantity is distributed as a $\chi^2(p)$ random variable under $\mathcal{H}_0$.

For the total memory, we compute
$$ \bar{d} = \mathbbm{1}^T \hat{d}_{\text{GSE}},$$
and in the main paper we have shown that this quantity is distributed as a $\mathcal{N}(0,\Omega/m)$ random variable when the true total memory $\bar{d_0} = 0$.

We simulate $n = 100$ realizations of length $T = 2^{16}$ from a standard Gaussian process (thus $d=0$) of dimension $p = 200$, computing both the Wald and total memory test statistics. In Figure \ref{fig:test_stats}, we plot the a comparison of the sample distribution of each test statistic against its asymptotic distribution over a range of values for $m$. For values of $m$ close to $p$, we see that the empirical type-I error of the Wald test is severely inflated relative to the nominal level $\alpha = 0.05$; in other words, the test spuriously rejects the null and claims to find long memory when none exists at a rate much higher than accounted for. The total memory test, by contrast, largely avoids this issue, even in the case where there are barely more observations than dimensions.

\begin{figure}[t!]
\centering
  \includegraphics[width=0.75\textwidth,height=12cm]{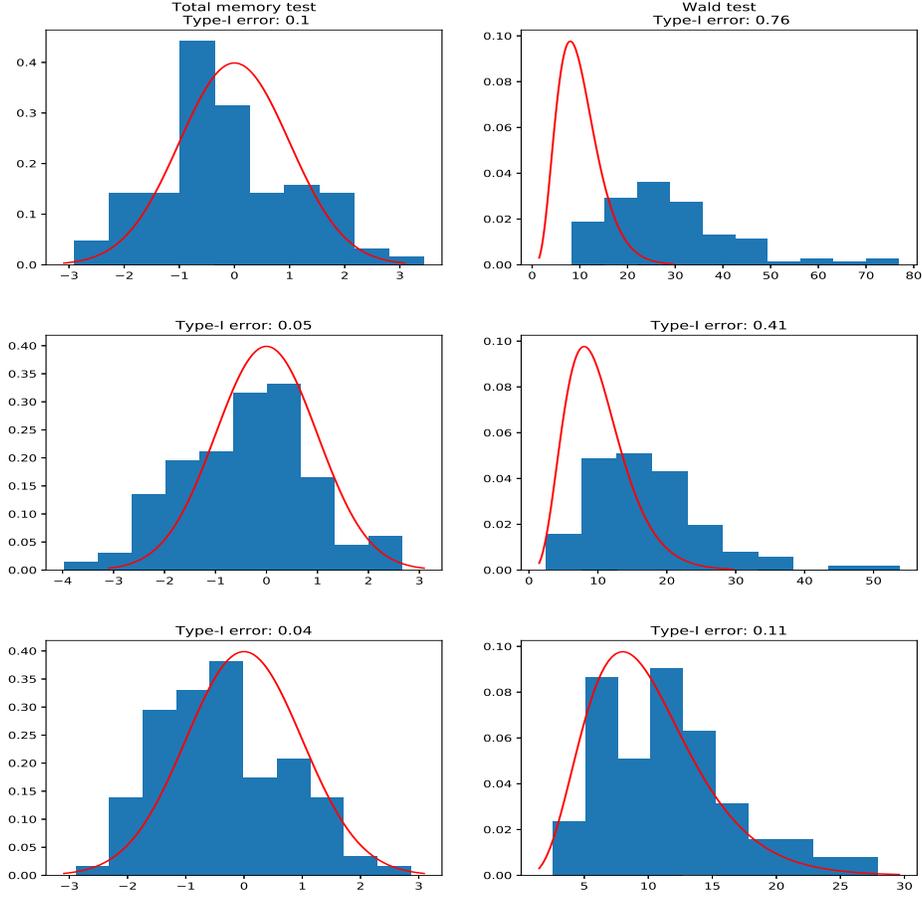}
  \caption{\small Sample distribution of the test statistic over $n = 100$ trials for $m = \sqrt{T} = 256$ (top row),  $m = 512$ (middle), and  $m = 1280$ (bottom). Empirical type-I errors are computed using the critical value corresponding to a nominal type-I error of $0.05$.}
  \label{fig:test_stats}
\end{figure}

Of course, with enough data, the Wald test becomes increasingly well-calibrated, but this is not at all an easy condition to satisfy while maintaining the integrity of the statistical analysis. We have already seen in Appendix C that simply increasing $m$ is not an option for real-world data, as this is likely to induce significant bias. On the other hand, the length $T$ of the observed sequence would have to be enormous, even by machine learning standards, to achieve $m \gg p$ with the reasonable choice $m = \sqrt{T}$ when the dimension $p$ is large. Finally, even if such data were available, we would likely prefer a method that allows valid inference at lower $m$ for computational reasons.
\section{Impact of embedding choice on long memory}

We evaluate the impact of embedding choice on estimated long memory from two perspectives. First, we include a re-analysis of the Bach cello suite data using the same MFCC features as used for the Miles Davis and Oum Kalthoum recordings. This allows us to state results for long memory estimation uniformly across a single choice of embedding, and to evaluate the impact of embedding choice on the long memory analysis across two very different but informative representations of the raw time series. The results (see Table \ref{tab:bach_embed}) show that the Bach data has long memory under both representations, though the average strength as measured by normalized total memory is somewhat variable. 

{\renewcommand{\arraystretch}{1.3}
\begin{table}[h!]
\small 
\captionsetup{size=small}
\setlength{\tabcolsep}{3pt} 
\caption{Long memory of Bach data by choice of embedding.}
\begin{tabularx}{\columnwidth}{@{} *{1}{C} *{1}{C} *{1}{C} c @{}}
    \toprule
     {\bf Embedding} & {\bf Norm. total memory} & {\bf p-value} & {\bf Reject $\mathcal{H}_0?$} \\
   \hline
Mel-frequency cepstral coefficients & 0.308 & 0.003 & \checkmark \\ 
Convolutional features & 0.0997 & $<  1 \times 10^{-16}$ & \checkmark \\
    \bottomrule
\end{tabularx}
\label{tab:bach_embed}
\end{table}
}

\begin{figure}[h!]
\centering
  \includegraphics[width=0.45\textwidth,height=6cm]{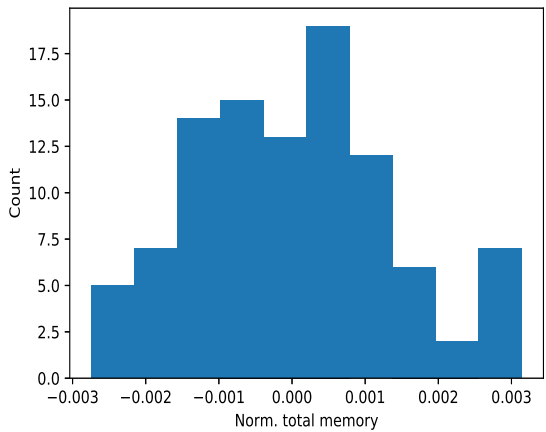}
  \caption{\small Histogram of normalized total memory computed from $n = 100$ permutations of the Penn TreeBank training data.}
  \label{fig:perm}
\end{figure}

Second, we consider a ``negative control" experiment in which we re-estimate the long memory vector for the embedded Penn TreeBank training set after permuting the sequential ordering of the data. This addresses the question of whether our positive result truly captures a sequence-dependent property of the data, or if it could have been produced spuriously as the consequence of other decisions related to the data analysis (including, for example, the choice of embedding). We compute the total memory statistic for $n=100$ random permutations of the Penn TreeBank training data. 

The results (see Figure \ref{fig:perm}) show that the total memory of the permuted data is concentrated near zero, with a sample mean of $1.90 \times 10^{-5}$ and standard error $0.00136$; a one-sample test of the mean correspondingly fails to reject the null hypothesis $\mathcal{H}_0: \bar{d} = 0$ with $p=0.494$.

\section{Classical music features from a MusicNet convolutional model} 

The reduced version of the MusicNet model of \cite{thickstun2} used to obtain an embedding for the Bach cello suite is derived from the convolutional model implemented in \texttt{musicnet\_module.ipynb}, a PyTorch interface to MusicNet available at \url{https://github.com/jthickstun/pytorch\_musicnet}. We reduce the number of hidden states to $200$ (this corresponds to setting $k=200$ in the notebook), both for computational tractability in the optimization procedure and to achieve consistency with the embedding dimension for our natural language experiments.

The model is trained on the MusicNet training corpus with no further modification of the tutorial notebook. Successful training and an informative feature mapping are indicated by the competitive performance of the model, even despite the reduced dimension of the hidden representation, in terms of the average precision of its predictions on the test set (see Table \ref{tab:music}). Results for our trained model (\emph{longmem-embed}) are favorable in comparison to both short-time Fourier transform (STFT) and commercial software (Melodyne) baselines, while approaching the quality of the fully learned filterbank (Learned filterbank; \citet{thickstun1}) and state-of-the-art translation invariant network (Wide-translation-invariant; \citet{thickstun2}).

{\renewcommand{\arraystretch}{1.3}
\begin{table}[h!]
\center
\small 
\captionsetup{size=small}
\setlength{\tabcolsep}{3pt} 
\caption{Performance Comparison for Models of MusicNet Data} 
\begin{tabularx}{0.4\textwidth}{ c c }
    \toprule
     {\bf Model} & {\bf Avg. Precision} \\
    \cmidrule(l){1-2}
     STFT & 60.4 \\
     Melodyne & 58.8 \\
     \emph{longmem-embed} & 65.1 \\
     Learned filterbank & 67.8 \\
     Wide-translation-invariant & 77.3 \\
    \bottomrule
\end{tabularx}
\label{tab:music}
\end{table}
}

\end{document}